\documentclass[lettersize,journal]{IEEEtran}
\usepackage{tikz}
\usetikzlibrary{positioning, arrows.meta}
\usepackage{algpseudocode}
\usepackage[utf8]{inputenc}
\usepackage{amsmath,amsfonts}
\usepackage{algorithm,algpseudocode}

\usepackage{lipsum} 

\usepackage{enumitem}
\newlist{hlist}{itemize}{1}
\setlist[hlist,1]{
    label=\textbullet,
    labelsep=1em,  
    leftmargin=0pt,
    itemindent=1.5em,
    listparindent=0pt,
    itemsep=1em,
    align=left
}
\makeatletter
\renewcommand\subsubsection{\@startsection{subsubsection}{3}{\z@}%
                                     {1.5ex plus .1ex minus .2ex}%
                                     {0.5ex plus .1ex}%
                                     {\itshape\small\mdseries}}
\makeatother

\usepackage{array}
\usepackage[caption=false,font=normalsize,labelfont=sf,textfont=sf]{subfig}
\usepackage{textcomp}
\usepackage{stfloats}
\usepackage{url}
\usepackage{verbatim}
\usepackage{mathrsfs}
\usepackage{graphicx}
\usepackage{cite}
\usepackage{makecell}
\usepackage{textcomp}
\usepackage{array} 
\usepackage{longtable}
\usepackage{booktabs}
\usepackage{caption}
\usepackage{siunitx} 
\usepackage{float}
\usepackage{multirow}
\newcommand{\domainbed}{\textsc{DomainBed}\xspace}
\usepackage{tikz}
\usetikzlibrary{arrows.meta} 
\usepackage{graphicx} 
\usepackage{xcolor} 
\usepackage{xspace}

\usetikzlibrary{decorations.pathreplacing} 
\definecolor{lightgreen}{HTML}{e6ffe6} 
\definecolor{edgegreen}{HTML}{00cc00} 
\definecolor{contextencoderbrown}{HTML}{ffcc66} 
\definecolor{edgeencoderbrown}{HTML}{ffaa00} 
\definecolor{dcblue}{HTML}{99d6ff}
\definecolor{edgedcblue}{HTML}{008ae6}
\definecolor{classifiergrey}{HTML}{cccccc}
\definecolor{edgeclassifiergrey}{HTML}{666666}
\definecolor{efmyellow}{HTML}{cccc00}
\definecolor{fmyellow}{HTML}{ffff00}
\definecolor{efdpink}{HTML}{ff6666}
\definecolor{fdpink}{HTML}{ffb3b3}
\definecolor{econditionalblue}{HTML}{999999}
\definecolor{conditionalblue}{HTML}{ffffe6}
\definecolor{eadporange}{HTML}{99dfff}
\definecolor{adporange}{HTML}{e6f7ff}
\definecolor{erevpurple}{HTML}{ffd11a}
\definecolor{revpurple}{HTML}{ffe066}
\definecolor{grey}{HTML}{bfbfbf}
\definecolor{egrey}{HTML}{f2f2f2}
\tikzset{myarrow/.style={-{Triangle[length=2mm,width=2mm]}, line width=2pt, draw=lightblue, fill=lightblue}}
\definecolor{lightblue}{HTML}{66a3ff}

\usepackage[colorlinks=true,
            linkcolor=black,
            urlcolor=black,
            citecolor=black]{hyperref}
\makeatletter
\newcommand*\bigcdot{\mathpalette\bigcdot@{.5}}
\newcommand*\bigcdot@[2]{\mathbin{\vcenter{\hbox{\scalebox{#2}{$\m@th#1\bullet$}}}}}
\makeatother

\usepackage{amsmath}
\newtheorem{proof}{\textbf{Proof}}
\newtheorem{proofsketch}{\textbf{Proof sketch}}
\newtheorem{property}{\textbf{Property}}
\newtheorem{theorem}{\textbf{Theorem}}
\newtheorem{lemma}{\textbf{Lemma}}
\newtheorem{definition}{\textbf{Definition}}
\newtheorem{assumption}{\textbf{Assumption}}

\begin{document}

\title{C$^3$DG: Conditional Domain Generalization for Hyperspectral Imagery
Classification with Convergence and Constrained-risk Theories }

\author{Zhe Gao, Bin Pan, Zhenwei Shi
\thanks{This work was supported by the National Key R\&D Program of China
under the Grant 2022ZD0160401 and 2022YFA1003800, the National Natural Science Foundation
of China under the Grant 62001251, 62001252 and 62125102, and the
Beijing-Tianjin-Hebei Basic Research Cooperation Project under the Grant
F2021203109. (Corresponding author: Bin Pan)}
\thanks{Zhe Gao and Bin Pan are with the School of Statistics and Data Science,
KLMDASR, LEBPS, and LPMC, Nankai University, Tianjin 300071, China (email: \href{mailto:gzben01@gmail.com}{gzben01@gmail.com}; \href{mailto:panbin@nankai.edu.cn}{panbin@nankai.edu.cn}).}
\thanks{Zhenwei Shi is with Image Processing Center, School of Astronautics, Beihang University, Beijing 100191, China (e-mail: \href{mailto:shizhenwei@buaa.edu.cn}{shizhenwei@buaa.edu.cn}).}
}

\markboth{IEEE Transactions on Cybernetics}%
{Shell \MakeLowercase{\textit{et al.}}: A Sample Article Using IEEEtran.cls for IEEE Journals}

\IEEEpubid{}

\maketitle

\begin{abstract}
Hyperspectral imagery (HSI) classification contends with hyperspectral-monospectra, where different classes reflect similar spectra. 
Mainstream leveraging spatial features but inflate accuracy for inevitably including test pixels in training patches. Domain generalization shows potential but approaches still fail to distinguish similar spectra across varying domains and overlook the theoretical support. Therefore, we focus solely on spectral information and propose a  Convergence and Constrained-error Conditional Domain Generalization method for Hyperspectral Imagery Classification (C$^3$DG) that can produce different outputs for similar inputs but varying domains.
Inspired by test time training, we introduce the Conditional Revising Inference Block (CRIB) and provide theories with proofs for model convergence and generalization errors. 
CRIB employs a shared encoder and multi-branch decoders to fully leverage the conditional distribution during training, achieving a decoupling that aligns with the generation mechanisms of HSI. To ensure model convergence and maintain highly controllable error, we focus on two theorems. First, in the convergence corollary, we ensure optimization convergence by demonstrating that the gradients of the loss terms are not contradictory. Second, in the risk upper bound theorem, our theoretical analysis explores the relationship between test-time training and recent related work to establish a concrete bound for error. Experimental results on three benchmark datasets confirm the superiority of our proposed C$^3$DG.
%

\end{abstract}


\section{Introduction}

\IEEEPARstart{H}{yperspectral} imagery (HSI) contains rich spectral information across a wide range of wavelengths, enabling precise material classification. Consequently, HSI classification is crucial for various applications, including agricultural monitoring, environmental observation, and mineralogy \cite{yongtu1,seaice_renpeng,yongtu2,gaolianru_buguniao}. 
However, a significant challenge for HSI classification is the phenomenon of \textit{hyperspectral-monospectra} \cite{tongpuyiwu, tongpuyiwu2, tongwuyipu3,former3}, that same spectrum but reflected from different classes in hyperspectral datasets. This phenomenon severely impacts the accuracy of HSI classification and the applicability of standard models.

Recent hyperspectral classification methods incorporate spatial features to address the phenomenon of hyperspectral-monospectra.
Commonly employed algorithms include those based on autoencoders for encoder-decoder networks \cite{gaolianru_sae,DFE_AE1,yeyuanxin_sae} that enhance feature extraction effectiveness, filter methods \cite{tcyb_ensb1,tcyb_ensb2} that extract spatial features through various kernels and ensemble learning, unsupervised HSI classification methods \cite{zhongyanfei_unsup, unsup_tcyb, pengjiangtao_contrastive,fangleyuan_unsupervised,xumeng_jiasen_unsup} that capitalize on the plethora of unlabeled samples in HSI to improve classification accuracy, semi-supervised approaches \cite{semi1_tcyb, semi2_tcyb, semi3_tcyb} that simultaneously utilize a small amount of labeled data and a large volume of unlabeled data during the training stage, and band selection or reduction approaches \cite{lihengchao_bandreduction,reduc_tcyb,wangqi_reduction, zhongyanfei_bandselect,8602462} that focus on filtering out bands that are truly useful for classification.
From a structural perspective, various convolutional neural network architectures tailored for HSI classification tasks \cite{cnn1,xumengjisen_lwcnn} effectively utilize both spatial and spectral information. Graph-based approaches \cite{pengjiangtao_graph,wangqi_graph,lihengchao_graph,9626567} process spectral information through a global-local perspective, efficiently exploring the correlation between both adjacent and non-adjacent pixels. Recent transformer-based methods \cite{former1, 9440852, yeyuanxin_former} segment HSI into patches according to spatial and spectral criteria, similarly leveraging spatial and spectral information simultaneously.

As highlighted by recent studies \cite{rit18}, the extraction of spatial features can lead to falsely inflating accuracy for test samples participating in training. Moreover, hyperspectral-monospectra persists even when spatial information is taken into account \cite{former3}. For example, cement roads and rooftops from different elevations share the same spatial-spectral features. In addition, spatial-spectral algorithms fail when data is limited to spectral-only information \cite{eurocrops}. Therefore, the research on approaches that solely rely on spectral data attracts scholarly focus. 
Domain generalization \cite{DGreview} is a promising solution for addressing the hyperspectral-monospectra phenomenon. These approaches \cite{dubois2021optimal,chevalley2022invariant,mtl,kim2021selfreg} effectively extract domain features during feature extraction, allowing classifiers to differentiate between similar samples that have different domain characteristics, leading to divergent classification outcomes. Specifically, test time training \cite{new_arm,TTT} domain generalization methods generate distinct outputs for magnitude-similar samples from different domains through adjustments made at test time via a context network.
Recent research on using domain generalization methods in HSI classification primarily focuses on fixed-output domain generalization approaches, which can be categorized into two main streams. Most recent methods \cite{LDGnet, single-DG, s2ecnet} utilize single-source domain generalization that leverages a domain-generative data augmentation to extract domain-invariant features. Multi-source-domain methods \cite{mdgnet,lica} manually select domains and minimize cross-domain differences while maximizing intra-class variability to improve generalizability. However, these approaches still require spatial information to address the issue of hyperspectral-monospectra. However, there are still two main obstacles:

\begin{itemize}[leftmargin=*]

\item  As previously emphasized, these methods produce fixed outputs and are not specifically designed to address the hyperspectral-monospectra issue.

\item  Most of these methods focus on technical details and neglect theoretical support, leading to a lack of convergence analysis and constraints on generalization errors that may lead to conflicts among components within the method and subsequently affect the optimization.

\end{itemize}

Therefore, we propose a \textbf{C}onvergence and error-\textbf{C}onstrained \textbf{C}onditional \textbf{D}omain \textbf{G}eneralization (C$^3$DG) for HSI classification that can output divergent results for similar samples but with different domain features to address the hyperspectral-monospectra issue. Inspired by test time training, we introduce a \textbf{C}onditional \textbf{R}evising \textbf{I}nference \textbf{B}lock (CRIB) and provide proofs for model training convergence and the upper bound of generalization errors. CRIB utilizes a common decoupled encoder and a multi-branch domain information-aware decoder to achieve identifiable disentanglement and fully utilize the conditional distribution of the inputs (i.e., label information) during training. In the convergence corollary section, we ensure optimization convergence by proving that the inner product between the gradients of the loss terms is consistently positive. In the upper bound theorem section, Our theoretical analysis identifies the correlation between test-time training methods and recent sharpness-aware minimization theories to construct a constrained risk bound for the theorem.

\begin{itemize}[leftmargin=*]
    \item We propose a conditional domain generalization HSI classification strategy that addresses the hyperspectral-monospectra issue by outputting adapted results at the inference stage.
    \item We design a multi-branch conditional revising inference block to comprehensively utilize the conditional distributions of the input samples.
    \item We provide theoretical analysis for the proposed method. One corollary focuses on the convergence of the model, and the other theorem constrains the upper bound of the classification risk.
\end{itemize}




\section{Methodology}

\subsection{Backgrounds}
Most models encounter significant accuracy drops when applied to unseen target domains after being trained on source domains. Consequently, domain generalization, which aims to enhance a model’s generalization capability from source domains to other unseen domains, has attracted academic attention. A brief definition can be summarized as follows.

\noindent\textbf{Domain Generalization}\quad Given $M$ source domains with labels$~\{X_i\}_{i=1}^{M}$,$~X_i=\{x_j,y_j\}_{j=1}^{N_i}$. $N_i$ is the number of samples in domain $X_i$. $x_j$, $y_j$ are the data and labels. The aim is to obtain a generalizable model for the unseen target domains with$~\{X_i\}_{i=1}^{M}$, formally:
\begin{equation}
    \min \mathbb{E}_{D \in \mathcal{D}} \left[ \mathbb{E}_{x \in D}\left[ \mathcal{R}(f(x),y) \right] \right]
\end{equation}
where $\mathcal{D}$ is the total domain space and $\mathcal{R}$ is risk function. 

Test time training methods of domain generalization aim to learn models that adapt at testing to domain shift using unlabeled test data, which can be formally defined as follows.
\noindent\textbf{Test Time Training} aims to learn a function shift from the input batch $\mathcal{B}$ and current model parameters $\theta$ to the C$^3$DG network parameter $\theta^{'}$, i.e. $f_{shift}: \mathcal{B},\theta \rightarrow \theta^{'}$, that satisfy the following optimization problem.

\begin{equation}
    \min_{\theta, \phi} \hat{\mathcal{E}}(\theta, \phi) = \mathbb{E}_{p_z} \left[ \mathbb{E}_{p_{xy|z}} \left[ \frac{1}{K} \sum_{k=1}^{K} \mathcal{R} (f(x_k; \theta'), y_k) \right] \right]
\end{equation}
where $\theta^{'} = f_{shift}(\theta, x_1, \ldots, x_K; \phi)$ and $\phi$ the parameters of $f_{shift}$.

An indirectly but workable implementation approach \cite{new_arm} for $f_{shift}$ is as follows:
\begin{equation}
    f(x,\theta^{'})\simeq f(\text{concat}(x,x_{context})),
\end{equation}
where $x_{context} = f_{cont}(x)$ is the context feature and $f_{cont}$ the context network.
These settings regard the combining of the features with test data as the shift in the model parameters.

\subsection{The Overall Framework}
In this section, we formally introduce the C$^3$DG framework. Our method can be integrated with any hyperspectral classification approach as a baseline by changing the Backbone Classifier $F$ in the flowchart Fig. \ref{framework}. For any classifier $F$, our method serves as a domain feature extractor through an auxiliary branch of CRIB architecture, thereby being able to provide revised outputs for similar inputs but with different domain information. Specifically, We employ a simple 1D-CNN network, of which the extractor is of three 1dConv-layers and the classifier is a two-linearlayer structure ending with a softmax, as the baseline for convenience.

\begin{algorithm}
\caption{Training stage of C$^3$DG}
\begin{algorithmic}[1] 
\Require $M$ source domains with labels $\{X_i\}_{i=1}^{M}$,$X_i=\{x_j,y_j\}_{j=1}^{N_i}$
\Ensure The parameters $\theta_D, \theta_F, \theta_f$ of C$^3$DG $D$, backbone predictor $F$ and pseudo-classifier $f_{pse}$.

\State \textbf{Initialize:} $\theta_G, \theta_F, \theta_f$
\For{\textit{epoch} = 1 to $T$}
    \State \textbf{Predict pseudo auxiliary labels}
    \State $\hat{y}_{i_{pse}}=f_{pse}(x_i)$
    \State \textbf{Extract Domain Context Features}
    \State $(\hat{z}_{d_i},\hat{z}_{m_i})=\hat{f}(x_i)$
    \State $\hat{z}_{s_i}=\hat{f}_{y_i}(\hat{z}_{m_i})$
    \For{\texttt{$c$ = 1 to $C$}}
        \State $\mathbf{z}_{s_c}=\text{BI}(\underset{i \leq N}{\text{mean}}\{\hat{z}_{s_i}|y_i=c\})$ \scriptsize \# BI: Bilinear Interpolate
        \State \normalsize $\mathbf{x}=\text{concat}(\mathbf{x},\mathbf{z}_{s_c})$
    \EndFor
    \State \textbf{Predict the Labels }
    \State $\hat{\mathbf{y}}=F(\mathbf{x})$
    \State Calculate the loss $L_{recon}$ through Eq. (\ref{realLrecon}).
    \State Calculate the loss $L_{rev}$ through Eq. (\ref{lrev}).
    \State Calculate the loss $L_{cls}$ through Eq. (\ref{lcls}).
    \State Update the model by gradient descent
\EndFor
\end{algorithmic}
\end{algorithm}

As illustrated in Fig. \ref{framework}, spectra from the input batch of different pixels (denoted as $x_1,x_2...x_M$) are firstly fed into the CRIB branch to extract the domain feature of these spectra. Specifically, the context encoder $\hat{f}$ in the CRIB processes this batch of data points $\mathcal{B}$ unclasswisedly to obtain the domain variables $z_d$ and the mixed variables $z_m$. It is important to highlight that a simple pseudo-label $\hat{f}_{pse}$ classifier of three fully connected layers is included since the labels are unknown but necessary during the inference stage. A sample spectrum $x_i$ in the batch is processed based on its rough pseudo-label $y_{pse_i}$; that is, only the forward encoder $\hat{f}_{pse_i}$ of the $y_{pse_i}$-th branch is turned on to process $x_i$ to extract the class-independent variable $z_{s_i}$, which is then sorted into discrete sets $\mathcal{Z}_{s_1}=\{z_{s_i}~|~y_{pse_i} = 1\},...,\mathcal{Z}_{s_C}=\{z_{s_i}~|~y_{pse_i} = C\}$ ($C$ is the total class number). The output sets of all branches $\mathcal{Z}_{s_1},...,\mathcal{Z}_{s_C}$ are taken means set-wisely and then reshaped to match the size of the input spectra point. Sequently, the average reshaped $z_{s_1},...,z_{s_C}$ are concatenated with $x_1,x_2...x_M$ channel-wisely to form the final revision $\{\text{Concat}(x_1,z_{s_1},...,z_{s_C})\}_{i=1}^{M}$ before fedding into the backbone classifier $F$ to obtain the finnal outputs $\{\hat{y}_i\}_{i=1}^M$. Particularly, the real class labels are utilized for deciding the CRIB branches during the training phase.


\begin{figure*}
\centering
\label{framework}
\includegraphics[width=\textwidth]{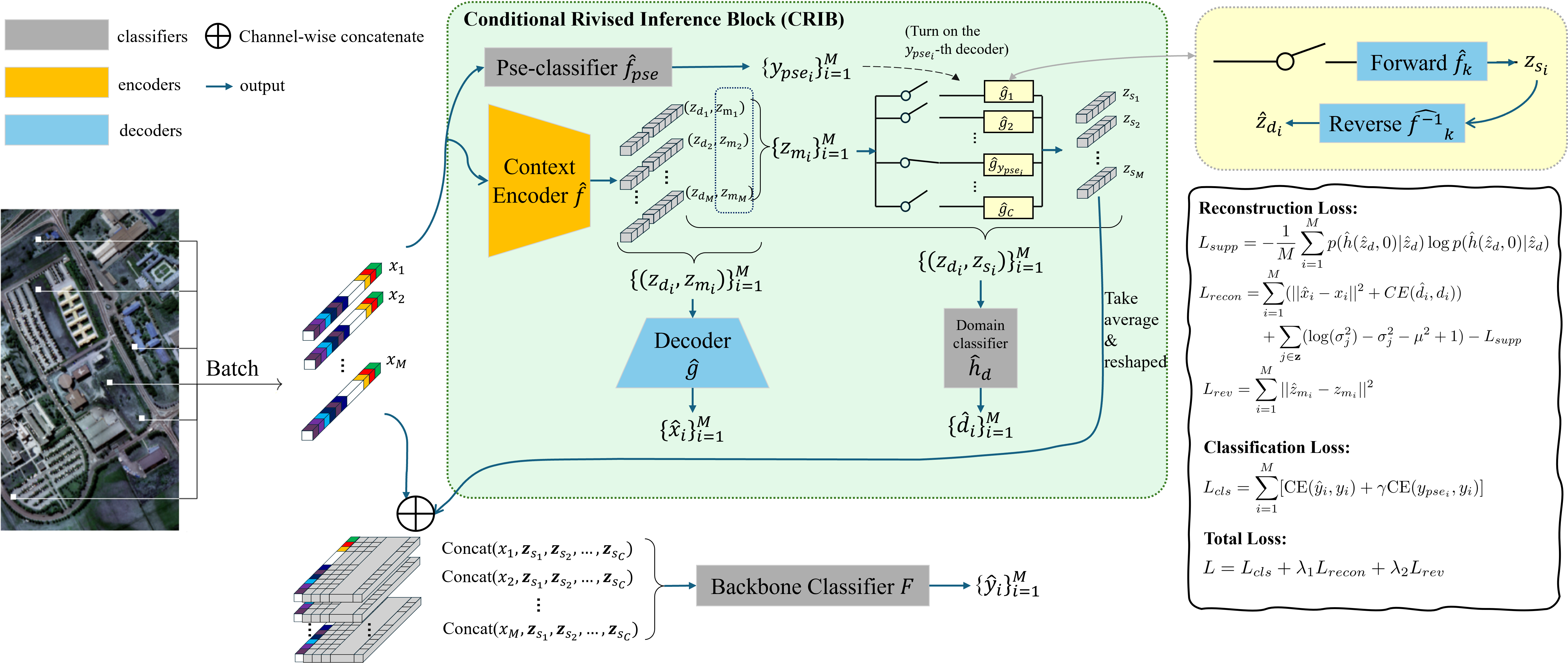}
\caption{The framework of proposed C$^3$DG architecture. The context encoder in the CRIB consists of C branches that handle different conditional distributions, sharing the same encoder. Specifically, C forward networks, named Reverse Dualnet, act as decoders to output the final revision features.}
\label{framework}
\end{figure*}

\subsection{Conditional Revising Inference Block}
\subsubsection{\textbf{Architectures of CRIB}}
A straightforward approach to characterize $\{P(X|c)\}_{c=1}^{C}$ is developing C branches of individual context networks tailored for each different conditional distribution. However, this implementation significantly enlarges the model scale since the complexity of context networks is comparably large and proportional to the usually extensive number of classes C. It is detrimental, not only augmenting the computational demands for training but also leading to a pronounced imbalance between the main classification and auxiliary tasks. This imbalance contradicts an empirical observation \cite{new_arm} that the auxiliary task should be subordinate and not undermine the main task. 

\begin{figure}
\centering
\includegraphics[width=8cm]{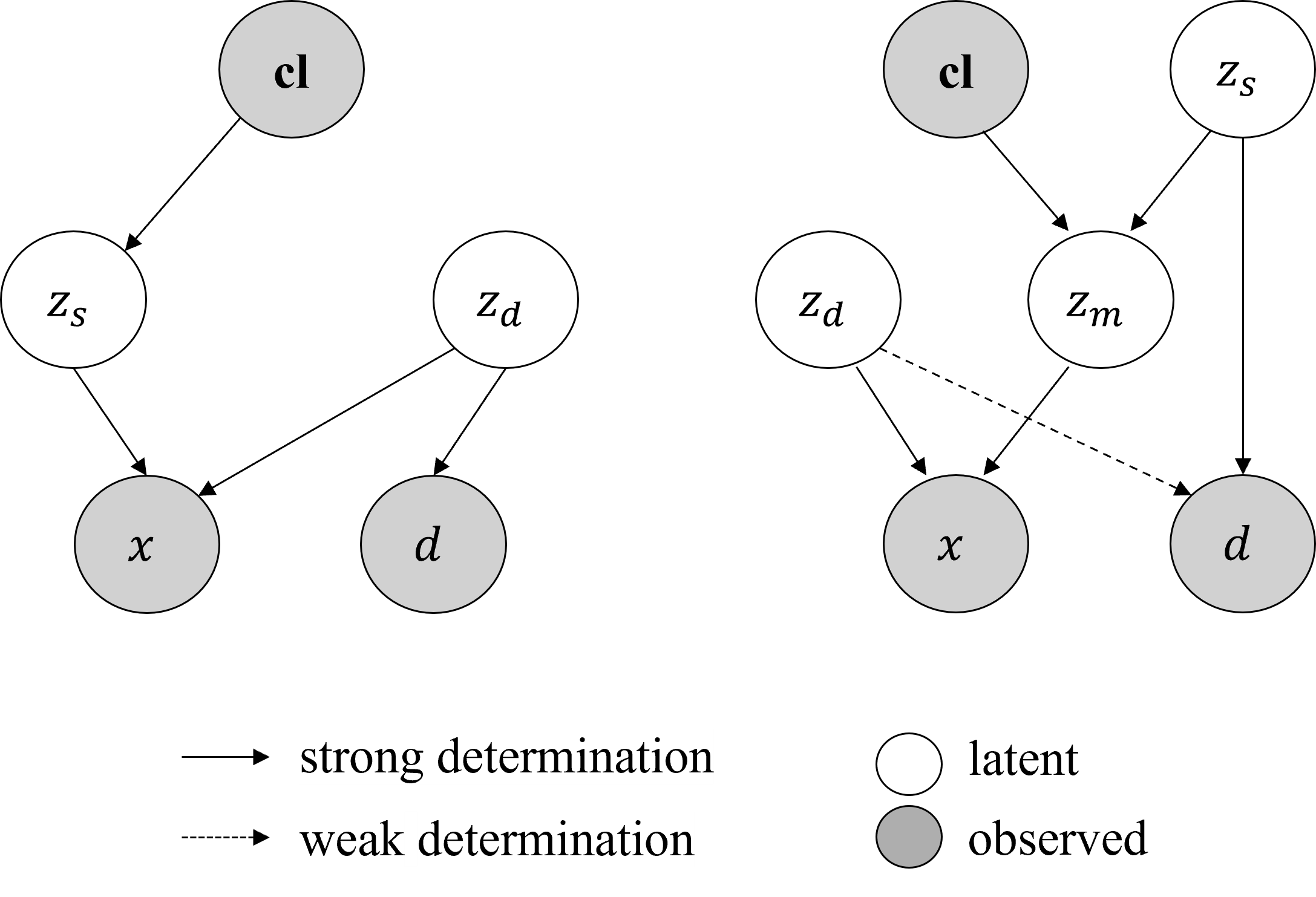}
\caption{Causal map of HSI generating. \textbf{left} is from previous work and \textbf{right} is our finer considerated map.}
\label{generatefig}
\end{figure}

To address the parameter redundancy arising from the multiple VAE branches, our proposed model ensures that most parameters are shared across these branches. Additionally, it utilizes divergent decoders, which allow for the production of distinct context features tailored to each class. This approach effectively balances the need for specificity in the model's outputs with the efficiency of parameter usage.

As depicted in Fig \ref{generatefig}, we formalize the inherent generative mechanism of generating the hyperspectral pixels through a two-stage causal diagram. Firstly, the class-independent variable $z_s$ and the label factor $c$ jointly determine the mixed variable $z_{m}$ which is crucial for label classification. Subsequently, the domain variable $z_{d}$, which could be factors such as sensor models or illumination intensity that cause inevitable hyperspectral-monospectra and are challenging to characterize, generated the real-life spectra merged with the mixed variable $z_{m}$. The comprehensive characterization $z_d, z_s$ of conditional distributions naturally suits the testing revision for the model's test stage. The essence of addressing hyperspectral-monospectra lies in the successful disentanglement of the input $x$.

Therefore, the CRIB, based on domain identification variational autoencoder, effectively extracts the feature $z_d$. As the causality map demonstrates, the ideal context networks for characterizing $\{P(X|c)\}_{c=1}^{C}$ share tremendous portions of their structures. The genuinely requisite architectures to be class-wisely customized is primarily the decoder (denoted as $\hat{f}_c$) that disentangles $z_m$ into $z_s$.

Our CRIB based on VAE structure accomplishes this disentanglement by utilizing domain labels in a classification task. Specifically, as illustrated in Fig. \ref{framework}, pixels are initially processed through the encoder $f_{\mu,\sigma}$ to obtain preliminary $z_d,z_m$. Then the $\{z_m\}$ are processed by different $\hat{f}_c$ depend on their labels. Ultimately, the fully-customized domain information $z_d, z_s$ is employed as the requisite context feature $z$.

\subsubsection{\textbf{Optimization for CRIB}}
\begin{algorithm}
\caption{Test stage of C$^3$DG}
\begin{algorithmic}[1] 
\Require A batch of $\{x_i\}_{i=1}^{M}$
\Ensure Predict labels $\{\hat{y}_i\}_{i=1}^{M}$.
    \State \textbf{Predict pseudo auxiliary labels}
    \State $\hat{y}_{i_{pse}}=f(x_i)$
    \State \textbf{Extract Domain Context Features}
    \State $(\hat{z}_{d_i},\hat{z}_{m_i})=\hat{f}(x_i)$
    \State $\hat{z}_{s_i}=\hat{f}_{\hat{y}_{i_{pse}}}(\hat{z}_{m_i})$
    \For{\texttt{$c$ = 1 to $C$}}
        \State $\mathbf{z}_{s_c}=\underset{i \leq N}{\text{mean}}\{\hat{z}_{s_i}|\hat{y}_{i_{pse}}=c\}$
        \State $\mathbf{x}=\text{concat}(\mathbf{x},\mathbf{z}_{s_c})$
    \EndFor
    \State \textbf{Predict the Labels }
    \State $\hat{\mathbf{y}}=F(\mathbf{x})$
\end{algorithmic}
\end{algorithm}
For the training of the public encoder and various decoders, a reconstruction ELBO loss is required both theoretically \cite{bvae} and intuitively justified. We adhere to the reconstruction loss $L_{recon}$ as part of the auxiliary task. This reconstruction loss ensures the extracting ability of the CRIB while mitigating the risk of overfitting. It is notable that $\hat{f}_u$ is supposed to be reversed in our settings, therefore a reversibility guarantee loss $L_{rev}$ is proposed with a pair of networks $\hat{f}_u, \hat{f}_u^{-1}$. We illustrate the shared decoder as $g$, the entire classification model as $F$, and the number of training points as $N$. $\beta$ is a hyperparameter. The following are the loss functions.
\begin{equation}
\begin{split}
\label{lrecon}
L_{recon} =& \sum_{i=1}^{N}(||g(z_{d_i},z_{m_i})-x||^2 +CE(\hat{h}_d(z_{d_i},z_{s_i}),d_i)) \\
&+\beta\text{KL}((z_m,z_s),N(0,1))
\end{split}
\end{equation}
\begin{equation}
\label{lrev}
L_{rev} = \sum_{i=1}^{N}||\hat{f}_{y_i}(\hat{f}_{y_i}^{-1}(z_{m_i}))-z_{m_i}||^2
\end{equation}
where KL is the KL divergence. 

As discussed in \cite{Wang_Xu_Ni_Zhang_2020}, we reinforce the disantanglement by mutual information suppression that minimizes $H(\hat{d}|z_d)$. The definition is as follows,
\begin{equation}
\label{hdd}
H(\hat{d}|z_d)=-\underset{\hat{z}_d,\hat{d}}{\sum}p(\hat{z}_d,\hat{d})\log \frac{p(\hat{z}_d,\hat{d})}{p(\hat{z}_d)}
\end{equation}
However, for the challenge of analyzing $p(\hat{z}_d,\hat{d})$, we adopt the following approximation to simulate $H(\hat{d}|z_d)$ and detailed theoretical support for the approximation is in the next section.
\begin{equation}
\label{Hhat}
L_{supp} = \hat{H}(\hat{d}|\hat{z}_d) := -\frac{1}{N}\sum_{i=1}^{N}p(\hat{h}(\hat{z}_d,0)|\hat{z}_d) \log p(\hat{h}(\hat{z}_d,0)|\hat{z}_d)
\end{equation}

By utilizing the closed solution of the KL divergence, the reconstruction loss $L_{recon}$ can be simplified as follows,
\begin{equation}
\label{realLrecon}
\begin{split}
L_{recon} =& \sum_{i=1}^{N}(||g(z_{d_i},z_{m_i})-x||^2 +CE(\hat{h}_d(z_{d_i},z_{s_i}),d_i)) \\
&+\underset{j \in \mathbf{z}}{\sum} (\log (\sigma_j^2)-\sigma_j^2-\mu^2+1) - L_{supp}\\
\end{split}
\end{equation}

The classification loss is as follows, where $CE$ denotes the cross-entropy loss and $\gamma$ is a hyperparameter.
\begin{equation}
\label{lcls}
L_{cls} = \sum_{i=1}^{N}[\text{CE}(F(x_i),y_i)+\gamma\text{CE}(f_{pse}(x_i),y_i)]
\end{equation}
From this, we derive the method's loss function:
\begin{equation}
\label{total_loss}
L = L_{cls}+\lambda_1L_{recon}+\lambda_2L_{rev}
\end{equation}
where $\lambda_1, \lambda_2$ are the hyperparameters.

\subsection{Theoretical Analysis for Convergence}
In this subsection, we first provide a specific definition of the hyperspectral-monospectra phenomenon from the perspectives of domain generalization and distribution. Then we present a corollary that ensures the convergence of the optimization. 

\begin{definition}
    \textbf{(hyperspectral-monospectra)} In hyperspectral images, pixels $x_1=x_2$ from different domains $D_1,D_2$ but belong to different $y_1,y_2$ labels. Such an issue is known as the hyperspectral-monospectra phenomenon.
\end{definition}

We will first provide an unbiased estimate of the aforementioned entropy loss term H to enable computability in practice.
\begin{property} 
The empirical expectation approximate
\begin{equation}
    -\frac{1}{N}\sum_{i=1}^{N}\mathbb{E}_{z_m \sim N(0,I)}p(h(\hat{z}_{d_i},z_m)|\hat{z}_i) \log p(h(\hat{z}_{d_i},z_m)|\hat{z}_i) 
\end{equation}
is unbiased estimation for $H(\hat{d}|z_d)$,
\end{property}

\begin{proofsketch}
\begin{equation}
    \begin{split}
        H(\hat{d}|\hat{z}_d) =&  -\frac{1}{N}\underset{\hat{z}_d,\hat{d}}{\sum}p(\hat{z}_d,\hat{d})\log \frac{p(\hat{z}_d,\hat{d})}{p(\hat{z}_d)}\\
        =& - \frac{1}{N}\underset{\hat{z}_d,z_m}{\sum}p(\hat{z}_d,\hat{h}(\hat{z}_d,z_m))\log \frac{p(\hat{z}_d,\hat{h}(\hat{z}_d,z_m))}{p(\hat{z}_d)}\\
        =&-\frac{1}{N}\sum_{i=1}^{N}\mathbb{E}_{z_m \sim N(0,I)}p(h(\hat{z}_{d_i},z_m)|\hat{z}_i)\\
        &~~~~~~~~~\cdot \log p(h(\hat{z}_{d_i},z_m)|\hat{z}_i)
    \end{split}
\end{equation}
\end{proofsketch}

We sample $z_m=0$ for computational convenience, thereby:
\begin{equation}
\begin{split}
     H(\hat{d}|\hat{z}_d) &\approx -\frac{1}{N}\sum_{i=1}^{N}p(\hat{h}(\hat{z}_d,0)|\hat{z}_d) \log p(\hat{h}(\hat{z}_d,0)|\hat{z}_d) \\
    &= \hat{H}(\hat{d}|\hat{z}_d)   
\end{split}
\end{equation}

\begin{theorem}
\label{coreTHRM}
\textbf{(The Convergency Theorem)} The backward gradient mutual information suppress loss is not contradict to the gradient of the backbone predictor, i.e. $\langle \frac{\partial\sum CE(\hat{h}_d(z_{d},z_{s}),d_i))}{\partial w_j},\frac{\partial\sum CE(\hat{h}_d(z_{d},z_{s}),d_i))}{\partial w_j} - \frac{\partial\sum \hat{d}_{d_i}log\hat{d}_{d_i}}{\partial w_j} \rangle \geq 0$, where $w_j$ is any parameter of the backbone predictor in C$^3$DG. 
Therefore, C$^3$DG leads to a convergence for the disentanglement.
\end{theorem}
\begin{proof}
\label{proof}
To simplify the problem, we consider the two-class classification problem (i.e. $d=2$). Also, take a one-layer classifier $\hat{h}_d$ for example, and ignore the softmax process. Let $\hat{d} = ReLU(WZ+B) \in \mathbb{R}^{d \times 1}$, where $W \in \mathbb{R}^{d \times k}, X \in \mathbb{R}^{k \times 1}$.

On the one hand:
\begin{equation}
    \begin{split}
        &\frac{\partial\sum CE(\hat{h}_d(z_{d},z_{s}),d_i))}{\partial w_j} \\
        =& \frac{\partial\sum CE(\hat{h}_d(z_{d},z_{s}),d_i))}{\partial \hat{d}} \cdot \frac{\partial \hat{d}}{\partial w_{m,1}} \\
        =& \frac{\partial \sum_{i = 1}^{N}   d_{i_j} \log \hat{d_{i_j}}+(1-d_{i_j})\log (1-\hat{d_{i_j}})}{\partial \hat{d}} \cdot \frac{\partial \hat{d}}{\partial w_{m,1}}       
    \end{split}
\end{equation}

On the other hand:
\begin{equation}
    \begin{split}
        &\frac{\partial\sum \hat{d}_{d_i}log\hat{d}_{d_i}}{\partial w_j} \\
        =& \frac{\partial \sum_{i = 1}^{N}   \hat{d}_{d_{i_j}} \log \hat{d}_{d_{i_j}}-(1-\hat{d}_{d_{i_j}}) \log (1-\hat{d}_{d_{i_j}})}{\partial \hat{d}} \cdot \frac{\partial \hat{d}}{\partial w_{m,1}}   
    \end{split}
\end{equation}

Therefore,
\begin{equation}
    \begin{split}
    &\frac{\partial\sum CE(\hat{h}_d(z_{d},z_{s}),d_i))}{\partial w_{m,1}} - \frac{\partial\sum \hat{d}_{d_i}log\hat{d}_{d_i}}{\partial w_{m,1}}\\
    =&\frac{\partial \sum_{i = 1}^{N}   (d_{i_ 1} \log \hat{d_{i_ 1}} - \hat{d}_{d_{d_ 1}} \log \hat{d}_{d_{d_ 1}})}{\partial \hat{d}} \cdot \frac{\partial \hat{d}}{\partial w_{m,1}}\\
    +&\frac{\partial \sum_{i = 1}^{N}   ((1-d_{i_ 1}) \log (1-\hat{d}_{i_ 1}) -(1 - \hat{d}_{d_{i_ 1}}) \log (1-\hat{d}_{d_{i_ 1}}))}{\partial \hat{d}}\\
    &\cdot \frac{\partial \hat{d}}{\partial w_{m,1}}
    \end{split}
\end{equation}

For $d = ReLU(WX+B)$, and suppose that  $\hat{d}_{i_ 1}, \hat{d}_{d_{i_ 1}} \in (0,1)$. Then we have,
\begin{equation}
    \begin{split}
        \frac{\partial \hat{d}}{\partial w_{m,1}}=\frac{\partial (WZ+B)}{\partial w_{m,1}}=\sum_{i=1}^{d}z_i
    \end{split}
\end{equation}

Therefore, $\langle \frac{\partial\sum CE(\hat{h}_d(z_{d},z_{s}),d_i))}{\partial w_j},\frac{\partial\sum CE(\hat{h}_d(z_{d},z_{s}),d_i))}{\partial w_j} - \frac{\partial\sum \hat{d}_{d_i}log\hat{d}_{d_i}}{\partial w_j} \rangle$ can be simpilified as follow.
    \begin{flalign*}
        =& (\sum_{i=1}^{n_z}z_i)^2\cdot \sum_{i=1}^{N} ((\frac{d_{i_1}}{\hat{d}_{i_1}}-\frac{1-d_{i_1}}{1-\hat{d}_{i_1}} )\cdot (\frac{d_{i_1}}{\hat{d}_{i_1}}-\frac{1-d_{i_1}}{1-\hat{d}_{i_1}}\\
        &+\log (\frac{1-\hat{d}_{d_1}}{\hat{d}_{d_1}})))\\
        =&(\sum_{i=1}^{n_z}z_i)^2\cdot\sum_{i=1}^{N} \frac{d_{i_1}-\hat{d}_{i_ 1}}{\hat{d}_{i_1}(1-\hat{d}_{i_1})} \cdot (\frac{d_{i_1}-\hat{d}_{i_ 1}}{\hat{d}_{i_1}(1-\hat{d}_{i_1})}+\log (\frac{1-\hat{d}_{d_1}}{\hat{d}_{d_1}}))\\
    \end{flalign*}

For $\hat{d}_{i_1}$ is a better prediction compared to $\hat{d}_{d_1}$, i.e. $0<\hat{d}_{d_1}<\hat{d}_{i_1}\approx1$, thereby $\frac{d_{i_1}-\hat{d}_{i_ 1}}{\hat{d}_{i_1}(1-\hat{d}_{i_1})} >> -\log (\frac{1-\hat{d}_{d_1}}{\hat{d}_{d_1}})$ holds, which means the inner product always bigger than 0.

\rightline{$\square$}
\end{proof}
\

\subsection{Theoretical Analysis for Upper Bound}
In this subsection, building upon the work of predecessors that focuses on sharpness-awareness minimization\cite{sam} and their several widely accepted assumptions, we will gradually derive the upper bound theorem for the proposed method.

In the subsequent discussion, we will first introduce a definition characterizing decoupling. Based on this definition, we also introduce two sets of assumptions that are commonly accepted within the research for identity disentanglement.

The following definitions in \cite{pmlr-v162-kong22a} and \cite{NEURIPS2023_6cb72460} characterize what is the ideal decoupling. 
\begin{definition}
    \label{PointIden}
    \textbf{(Partial Identifiability)} Let $z_d,z_m$ be the natural causal feature that jointly constructs the real data $(x,y)$, and $\hat{z}_d,\hat{z}_m$ be the approximate feature. Let $hom$ be the map from $z_d,z_m$ to $\hat{z}_d,\hat{z}_m$, i.e. $hom(z_d,z_m)=(\hat{z}_d,\hat{z}_m)$. $z_m$ is \textit{partial identifiable} if and only if $h_m$ is invertible, where $hom_m$ is the projection of $hom$ on the subspace of $\mathcal{Z}_m$
\end{definition}

We also adopt the assumptions and utilize the partial identifiability theories in \cite{NEURIPS2023_6cb72460}.
\begin{assumption}
\label{PIAssum}
\textbf{(Normal Generating Assumption)}
\begin{itemize}
    \item (Smooth and Positive Density): $p_{z|y}$ is smooth and $p_{z|y}>0$ over $\mathcal{Z}$ and $\mathcal{Y}$.
    \item (Conditional independence): Conditioned on u, each $z_i$ is independent of any other $z_j$ for $i,j \in \{1,2,...,n\}$. 
    \item (Class efficiency): Let vector $\mathbf{w}(\mathbf{z},\mathbf{y}):=(\frac{\partial logp_{z_i|y}(z_i|y)}{\partial z_i},...,\frac{\partial logp_{z_{n_s}|y}(z_{n_s}|y)}{\partial z_{n_s}})$. These $n_s$ vectors $\mathbf{w}(\mathbf{z},y_i)-\mathbf{w}(\mathbf{z},y_0)$ are linear independent.
\end{itemize}
\end{assumption}

\begin{lemma}
\label{lemmaPI}
\textbf{(Partial Identifiability and dimension sufficiency of $z_m$)} If the conditions in assumption \ref{PIAssum} holds, the mixed feature $z_m$ in fig. \ref{generatefig} is partial identifiable.
\end{lemma}

In ARM \cite{new_arm}, the authors propose various methods to simulate shifts in parameters, leading to the following assumptions:
\begin{assumption}
\label{SAMAssum}
\textbf{(Equivalent Shift Assumption)}
\begin{itemize}
 \item (Equivalent Parameter Adjustment) Combining extracted context features with the data $x$ is equivalent to directly modifying the model's parameters, i.e. $f_{\theta}((x,x_{context})) = f_{\theta^{'}}(x)$ where $\theta,\theta^{'}$ are the original and modified parameters of the network. 
 \item (Minuscule Parameter Adjustment) For any batch $B$, the Euclidean distance $|\theta-\theta^{'}|$ is controlled by a small constant $\gamma$
\end{itemize}    
\end{assumption}

Sharpness-awareness Maximization \cite{sam} has already demonstrated its enhancement of model generalization capabilities. Recent works \cite{pmlr-v162-kim22f} have provided a wealth of theoretical basis for DG.

\begin{lemma}
\label{lemmaSAM}
\textbf{(Better Generalbility of Adjusted Parameters)}
Fix the parameter $w$, and define the training set loss $L(w):=\frac{1}{n}\sum_{i=1}^{n} L(w,x_i,y_i)$ and real loss $L_D(w):= \mathbb{E}_{x,y} L(w,x,y)$. Now onsider the maximization $L(w+\epsilon_0)$ of problem $\underset{\| \epsilon /|w|\| \leq \rho }{\max} L(w+\epsilon)$. Let $w \in \mathbb{R}^k$, then the following inequality holds:

\begin{equation}
    \mathbb{E}_{\| \epsilon /|w|\| \leq \rho}L_{D}(w+\epsilon) \leq L(w+\epsilon_0)+\sqrt{\frac{O(k+log(\frac{n}{1-\delta}))}{n-1}}
\end{equation}
where $\delta$ is the scale ratio between training set $(X,Y)$ of size $n$ and the whole potential data space $(\mathcal{X},\mathcal{Y})$ of $N$, i.e. $\delta=\frac{n}{N}$.
\end{lemma}

\begin{theorem}
\label{coreTHRM}
\textbf{(The Upper Bound Theorem of Classification Risks)} Under assumption \ref{PIAssum} and \ref{SAMAssum}, The upper bound of the expected classification is as follows,
\begin{equation}
    \mathbb{E}_{\epsilon^{T}\hat{h}(\theta)\epsilon \leq \gamma}L_{D}(w+\epsilon) \leq L(w+\epsilon_0)+\sqrt{\frac{O(k+log(\frac{n}{1-\delta}))}{n-1}}
\end{equation}
\end{theorem}
\begin{proof}
\label{proof}
Let $p(y|x,\theta)$ and $p(y|x,\theta^{'})$ denote the ditributions of output from the original model and the C$^3$DG model. Denote $d(\theta,\theta^{'})=\mathbb{E}_{x}[\mathbf{KL}(p(y|x,\theta^{'})||p(y|x,\theta))]$, where $\mathbf{KL}$ is the KL-divergency.

For small $\epsilon$, we approximately have,

\begin{equation}
    d(\theta,\theta^{'}) \approx \epsilon^{T}F(\theta)\epsilon
\end{equation}
where $F(\theta)$ is the Fisher information matrix defined as follows,
\begin{equation}
    F(\theta)=\mathbb{E}_x \mathbb{E}_{\theta}[\nabla \log p(y|x,\theta)\nabla \log p(y|x,\theta)^{T}]
\end{equation}

We adopt the empirical diagonalised minibatch approximation form,
\begin{equation}
    \hat{F}(\theta)\approx\frac{1}{|B|}\sum diag(\nabla \log p(y_i|x_i,\theta))^2
\end{equation}

Now consider the maximization problem,
\begin{equation}
    \underset{\epsilon^{T}\hat{h}(\theta)\epsilon \leq \gamma}{\max} L(w+\epsilon)
\end{equation}
where $\gamma$ is the small hyperparameter in assumption \ref{SAMAssum}.

We now prove that $\epsilon_0 = \theta^{'}-\theta$ is one solution to the above problem. 
Now consider, 
\begin{equation}
    L(\theta^{'}) = \mathbb{E}_x [CE(y,p(y|x,\theta^{'}))] \approx \sum CE(y_i,p(y_i|x,\theta^{'}))
\end{equation}

Indeed, the main potential difference between the output distributions of $f_\theta,f_{\theta^{'}}$ is the “confusing confidence” points that diverge in the two classifiers. Based on the observation that the outputs change smoothly when $\theta$ perpetuates minusculely, in the manifold $\epsilon^{T}\hat{h}(\theta)\epsilon \leq \gamma$, $L(\theta)$ either stay steady or changes dramatically caused by aforementioned confusing points.

Therefore, by lemma \ref{lemmaSAM},
\begin{equation}
    \mathbb{E}_{\epsilon^{T}\hat{h}(\theta)\epsilon \leq \gamma}L_{D}(w+\epsilon) \leq L(w+\epsilon_0)+\sqrt{\frac{O(k+log(\frac{n}{1-\delta}))}{n-1}}
\end{equation}
where the definition of $\delta$ is the same as in lemma \ref{lemmaSAM}.

\rightline{$\square$}
\end{proof}
\

\section{Experiments}
\subsection{dataset}

\begin{figure}[htbp]
\centering
\includegraphics[width=\linewidth]{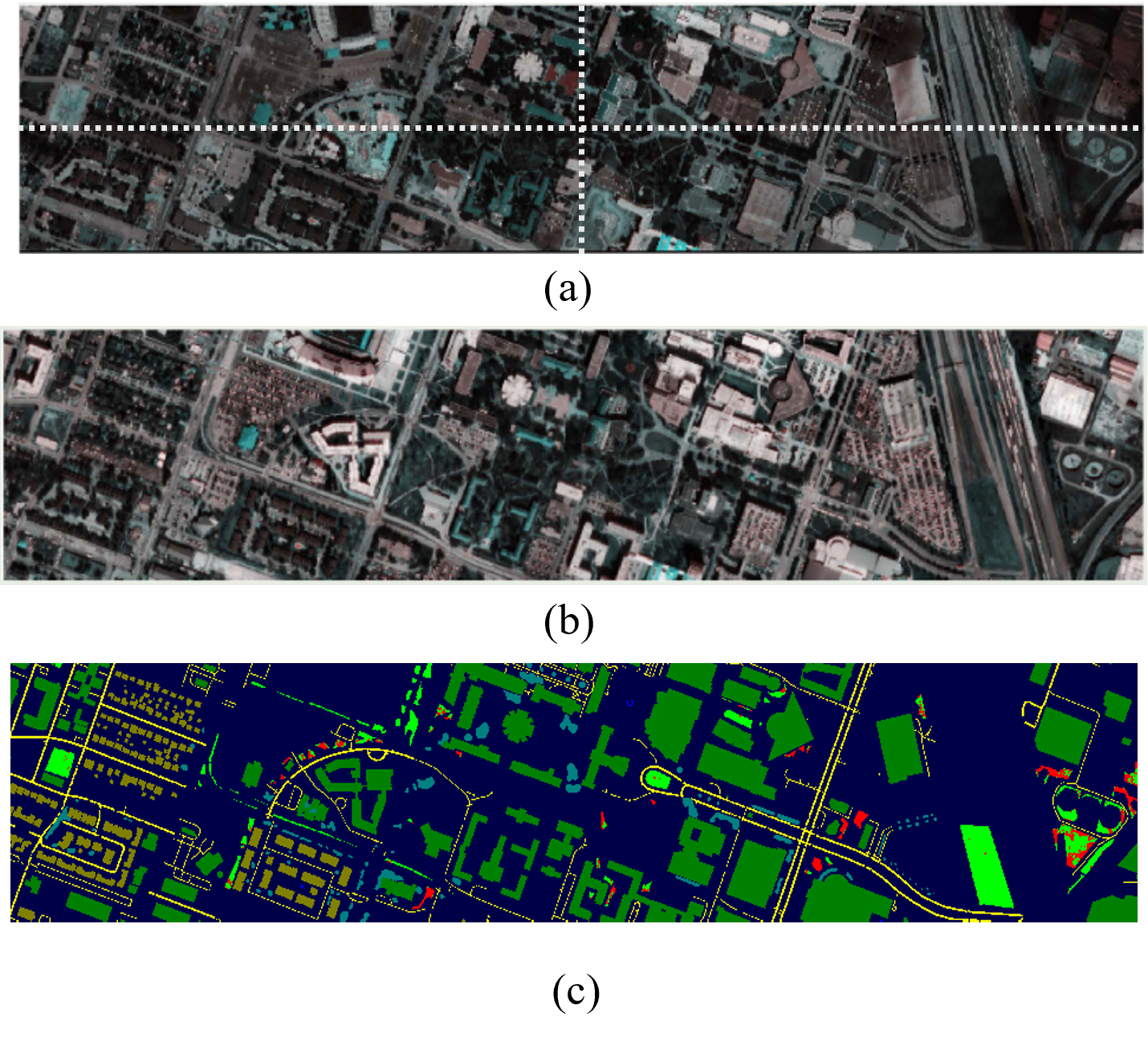}
\caption{(a) is the false-color image of the source domain Houston 2013, which we divided into four source domains along the white dashed lines. (b) is the false-color image of the target domain Houston 2018. (c) is the ground truth image of Houston 2018.}
\label{fig:gt_hu}
\end{figure}

\begin{figure}[htbp]
\centering
\includegraphics[width=\linewidth]{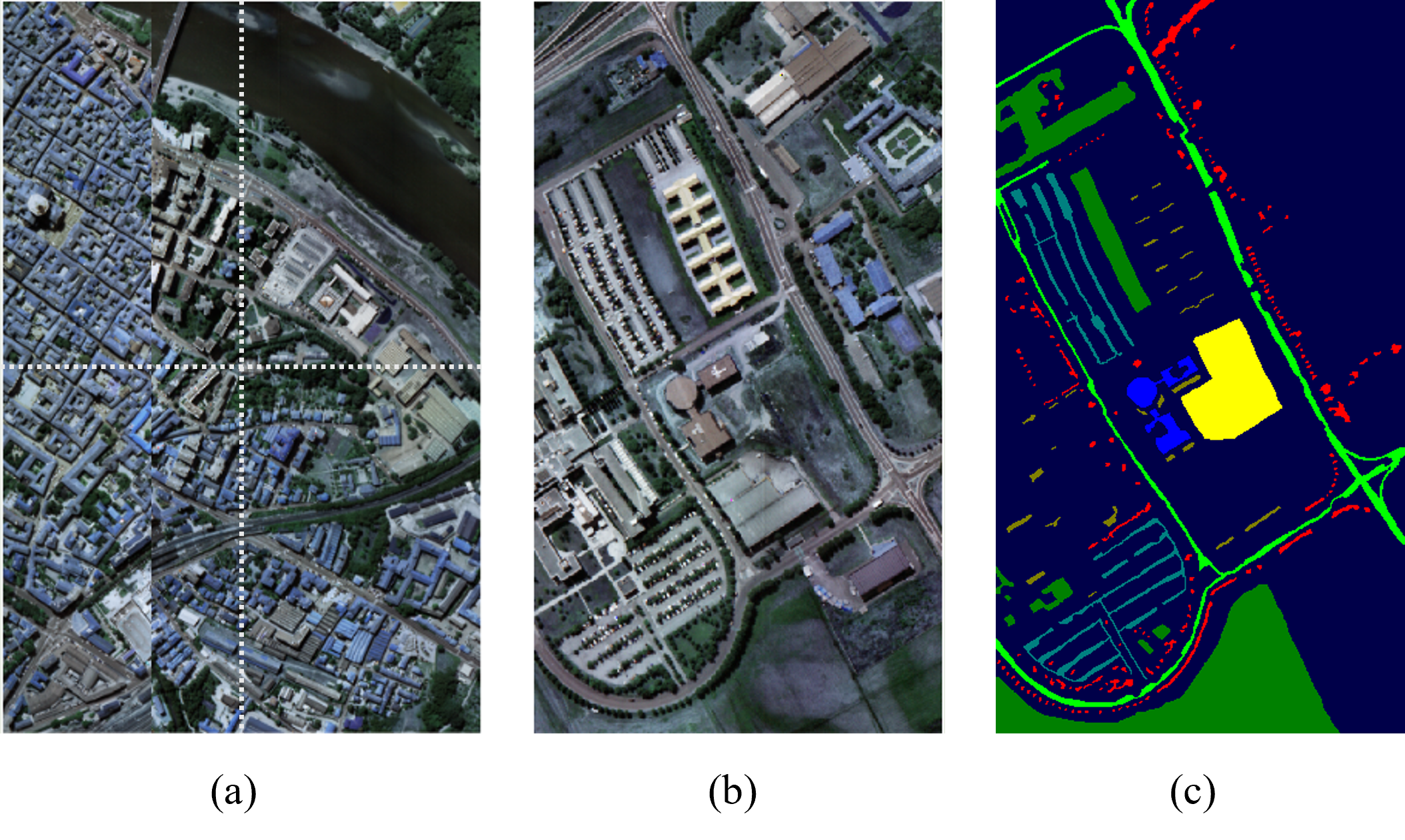}
\caption{(a) is the false-color image of the source domain Pavia City, which we divided into four source domains along the white dashed lines. (b) is the false-color image of the target domain Pavia University. (c) is the ground truth image of Pavia University.}
\label{fig:gt_cp}
\end{figure}

\begin{figure}[htbp]
\centering
\includegraphics[width=\linewidth]{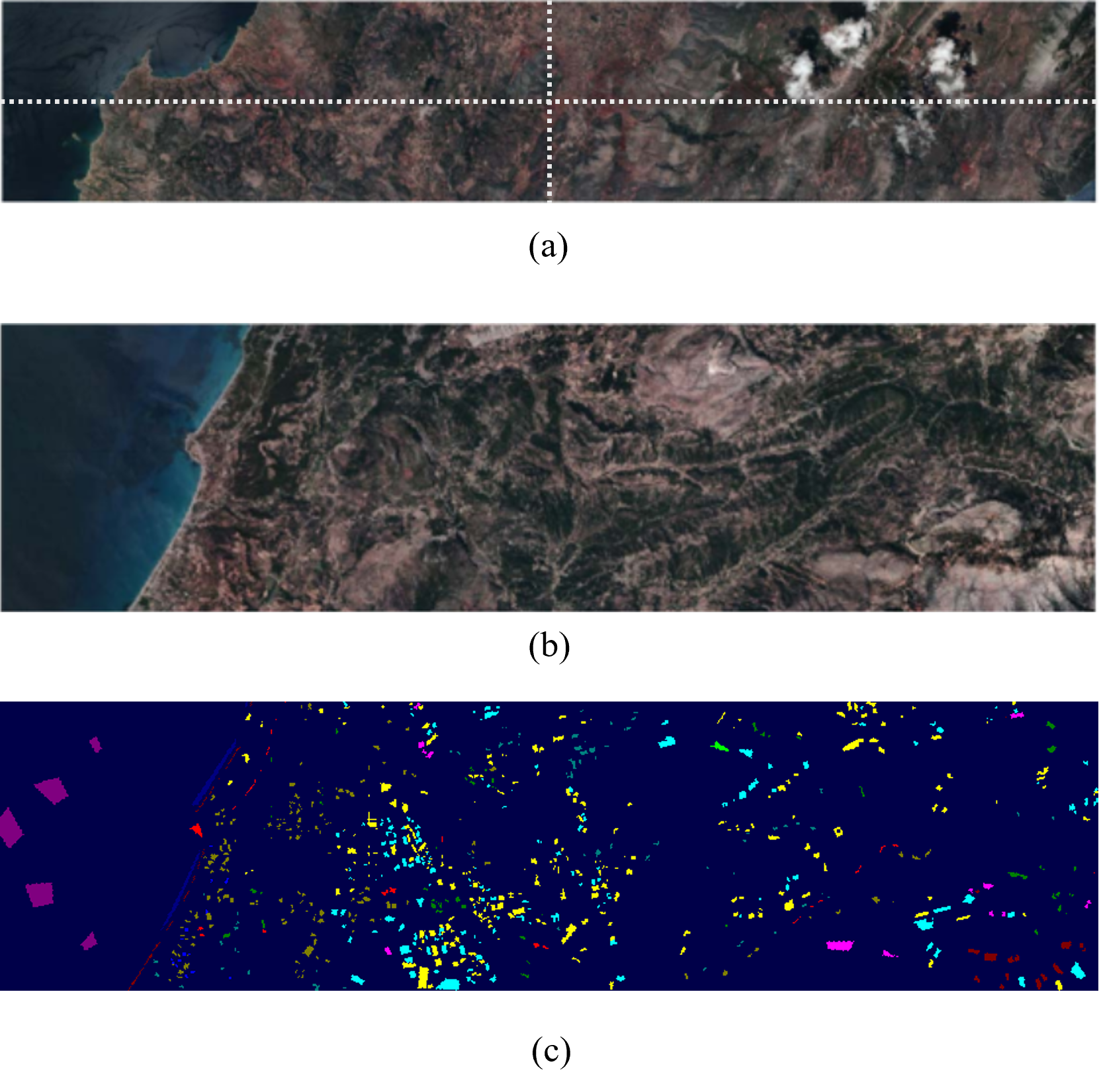}
\caption{(a) is the false-color image of the source domain Dioni, which we divided into four source domains along the white dashed lines. (b) is the false-color image of the target domain Loukia. (c) is the ground truth image of Loukia.}
\label{fig:gt_hr}
\end{figure}

\noindent\textbf{Cross-Scene Houston:} 
The dataset includes scenes from Houston 2013 and 2018, acquired through ITRES CASI-1500 located around the University of Houston campus. The Houston 2013 dataset is detailed with 349×1905 pixels, incorporating 144 spectral bands, spanning a wavelength range of 364-1046nm. Houston 2018 dataset retains the same wavelength range but a reduced count of 48 spectral bands. Both datasets contain seven consistent classes. For compatibility, 48 spectral bands within the wavelength range of 0.38-1.05µm are extracted from Houston 2013 to align with Houston 2018. This alignment focuses on a selected overlapping area measuring 209×955 pixels. We follow the experiment settings in \cite{single-DG} and visualize the datasets through false-color, ground truth maps, and category table, as depicted in Fig. \ref{fig:gt_hu} and Table \ref{sizehouston}.
 
\noindent\textbf{Cross-Scene Pavia:} 
The Pavia dataset comprises two scenes: Pavia University and Pavia Center, both captured by the Reflective Optics Spectrographic Image System (ROSIS). These datasets span a spectral range of 430nm-860nm. Pavia Center contains 1096×715 pixels and 102 spectral bands, while Pavia University includes 103 bands and 610×340 pixels. The last band of the Pavia University dataset was removed to ensure that both datasets share an identical number of spectral bands. These datasets include seven classes as listed in Table \ref{sizehouston}. Additionally, Fig. \ref{fig:gt_cp} illustrates the false-color images and ground-truth maps for both scenes.

\noindent\textbf{Hyrank Datasets:} 
The dataset encompasses hyperspectral imagery from two distinct regions, Dioni and Loukia. The dimensions of the Dioni region are 250×1376 pixels, whereas the Loukia region measures 249×945 pixels, with both areas captured at a spatial resolution of 30 meters. These samples were obtained using the Hyperion sensor, which records data across 176 spectral bands and identifies 14 consistent classes across the two regions. Data from all 176 spectral bands are utilized and only focus on 12 shared object categories. The specific count of samples across these categories is detailed in Table \ref{sizehouston}. We also provide false-color images and ground-truth maps as illustrated in Fig. \ref{fig:gt_hr}.

\subsection{Implementation Details}

In comparative experiments, we utilize the \textsc{domainbed}
implemented on the PyTorch framework. The training batch size was configured at $256$, while the test batch size was set at $64$. The learning rate was established at $0.005$. 

In our study, all encoder structures were based on 1D-CNNs, and all classifiers were implemented using fully connected neural network architectures. It is worth mentioning that, for computational convenience, our pseudo classifier was directly configured as a three-layer fully connected neural network.

For hyperparameters, both $\lambda_1$ and $\lambda_2$ were set at $0.1$, and the dimensions for $z_d, z_s,$ and $z_m$ were all set to 4. The update of the proposed method utilized the Adam optimizer, with weight decay included to prevent overfitting. Moreover, all experiments were run on a single RTX 3060 GPU. Each experiment was conducted eight times, with the results averaged to ensure reliability and accuracy.

\subsection{Comparative Experiments}

\begin{table*}[htbp]
\centering
\caption{COMPARISON OF DIFFERENT METHODS ON CROSS-SCENE HOUSTON DATASET.}
\begin{tabular}{
    S[detect-mode] 
    l 
    *{9}{S[table-format=2.2]} 
    }
\toprule[1.3pt]
\midrule[0.5pt]
{No.} & {Class} & {ERM} & {CAD \cite{dubois2021optimal}} & {CausIRL \cite{chevalley2022invariant}} & {MTL \cite{mtl}} & {SelfReg \cite{kim2021selfreg}} & {SDENet\_spe \cite{single-DG}} & {LiCa \cite{lica}} & {S$^2$ECNet\_spe \cite{s2ecnet}}&{C$^3$DG} \\
\midrule
\text{C1} & {Grass H.} & 37.83 & 93.50 & 85.29 & 65.07 & 40.07 & 94.47 & 86.94 & 70.89 & 90.37\\
\text{C2} & {Grass S.} & 26.40 & 32.87 & 0 & 0 & 19.94 & 29.29 & 0 & 37.16 & 26.34\\
\text{C3} & {Trees} & 32.31 & 28.42 & 16.88 & 18.63 & 24.25 & 70.24 & 18.09 & 72.42 & 33.88\\
\text{C4} & {Water} & 66.67 & 66.67 & 80.95 & 100 & 57.14 & 95.23 & 100 & 100 & 90.47\\
\text{C5} & {R. buildings} & 54.42 & 52.94 & 31.42 & 25.09 & 49.49 & 79.04 & 16.10 & 74.92 & 64.19\\
\text{C6} & {N. buildings} & 90.05 & 86.39 & 81.59 & 81.34 & 86.52 & 62.73 & 84.35 & 72.78 & 85.39\\
\text{C7} & {Road} & 8.89 & 6.99 & 38.22 & 58.35 & 12.72 & 55.57 & 24.47 & 31.43 & 19.55\\
\midrule
& {OA (\%)} & 66.48 & 65.70 & 60.56 & 61.79 & 63.41 & 61.65 & 59.20 & 64.71 & \textbf{68.23} \\
& {AA (\%)} & 45.22 & 52.53 & 47.76 & 49.79 & 41.44 & \textbf{69.52} & 47.14 & 65.66 & 58.36  \\
& {Kappa$\times$100}  & 39.06 & 38.82 & 31.97 & 35.86 & 32.16 & 43.52 & 25.61 & \textbf{43.92} & 42.60  \\
\midrule[0.5pt]
\bottomrule[1.3pt]
\end{tabular}
\label{tab:comparisonhou}
\end{table*}

\begin{table*}[htbp]
\centering
\caption{COMPARISON OF DIFFERENT METHODS ON CROSS-SCENE PAVIA DATASET.}
\begin{tabular}{
    S[detect-mode] 
    l 
    *{9}{S[table-format=2.2]} 
    }
\toprule[1.3pt]
\midrule[0.5pt]
{No.} & {Class} & {ERM} & {CAD} & {CausIRL} & {MTL} & {SelfReg} & {SDENet\_spe} & {LiCa} & {S$^2$ECNet\_spe}&{C$^3$DG} \\
\midrule
\text{C1} & {Tree} & 73.82 & 64.51 & 85.16 & 84.57 & 77.54 & 92.96 & 94.15 & 83.88 & 72.48 \\
\text{C2} & {Asphalt} & 71.47 & 62.84 & 81.12 & 79.05 & 66.14 & 78.83 & 71.29 & 72.98 & 74.40 \\
\text{C3} & {Brick} & 18.46 & 12.86 & 1.26 & 0.27 & 39.14 & 52.37 & 1.70 & 55.96 & 23.51 \\
\text{C4} & {Bitumen} & 18.35 & 94.97 & 8.75 & 11.88 & 58.56 & 2.28 & 55.22 & 4.41 & 8.91 \\
\text{C5} & {Shadow} & 99.36 & 99.14 & 99.57 & 99.71 & 99.46 & 99.25 & 99.89 & 95.42 & 99.57 \\
\text{C6} & {Meadow} & 58.24 & 53.51 & 42.20 & 53.74 & 53.56 & 37.23 & 48.65 & 41.95 & 59.81 \\
\text{C7} & {Bare soil} & 58.78 & 52.59 & 81.78 & 70.36 & 58.63 & 86.06 & 74.22 & 91.24 & 61.24 \\

\midrule
& {OA (\%)} & 56.32 & 53.07 & 51.34 & 57.74 & 58.13 & 56.83 & 56.35 & 57.58 & \textbf{59.29} \\
& {AA (\%)} & 56.92 & 57.12 & 62.92 & 57.20 & \textbf{64.72} & 63.85 & 63.59  & 63.69& 57.13\\
& {Kappa$\times$100} & 44.72 & 42.61 & 41.44 & 45.62 & 46.07 & 46.38 & 45.51 & 41.92 & \textbf{46.51}\\
\midrule[0.5pt]
\bottomrule[1.3pt]
\end{tabular}
\label{tab:comparisonpa}
\end{table*}

\begin{table*}[htbp]
\centering
\caption{COMPARISON OF DIFFERENT METHODS ON HYRANK DATASET.}
\begin{tabular}{
    S[detect-mode] 
    l 
    *{9}{S[table-format=2.2]} 
    }
\toprule[1.3pt]
\midrule[0.5pt]
{No.} & {Class} & {ERM} & {CAD} & {CausIRL} & {MTL} & {SelfReg} & {SDENet\_spe} & {LiCa} & {S$^2$ECNet\_spe}&{C$^3$DG} \\
\midrule
\text{C1} & {D. U. Fabric} & 9.80 & 0.49 & 0 & 20.09 & 0 & 0.98 & 4.90 & 2.94 & 0\\
\text{C2} & {M. E. Sites} & 25.92 & 100 & 5.56 & 37.03 & 66.67 & 27.78 & 79.63 & 53.70 & 61.11\\
\text{C3} & {N. I. A. Land} & 37.20 & 63.98 & 65.40 & 55.21 & 72.75 & 1.65 & 38.62 & 6.87 & 45.02\\
\text{C4} & {Fruit Trees} & 20.25 & 10.12 & 50.63 & 29.11 & 22.78 & 0 & 62.03 & 0 & 63.29\\
\text{C5} & {Olive Groves} & 1.92 & 0.73 & 4.04 & 15.05 & 2.57 & 0.36 & 5.79 & 0 & 0\\
\text{C6} & {C. Forest} & 49.88 & 37.17 & 33.09 & 41.01 & 51.80 & 26.86 & 38.61 & 10.07 & 52.51\\
\text{C7} & {D. S. Vegetation} & 82.40 & 89.78 & 79.80 & 66.49 & 81.36 & 68.91 & 70.87 & 76.19 & 77.74\\
\text{C8} & {S. S. Vegetation} & 48.50 & 35.94 & 43.34 & 46.79 & 55.13 & 66.92 & 34.17 & 77.90 & 63.50\\
\text{C9} & {S. V. Areas} & 84.63 & 22.41 & 39.80 & 56.67 & 33.75 & 84.88 & 86.90 & 76.82 & 79.34\\
\text{C10} & {Rocks and Sand} & 51.00 & 3.11 & 12.47 & 54.78 & 45.43 & 0.45 & 24.72 & 0 & 45.21\\
\text{C11} & {Water} & 100 & 100 & 100 & 99.71 & 100 & 100 & 100 & 99.86 & 100\\
\text{C12} & {Coastal Water} & 100 & 100 & 100 & 100 & 100 & 100 & 100 & 100 & 100\\
\midrule
& {OA (\%)} & 62.43 & 57.02 & 58.99 & 61.04 & 63.08 & 58.42 & 56.25 & 61.56 & \textbf{64.60} \\
& {AA (\%)} & 50.96 & 46.98 & 44.50 & 51.79 & 52.69 & 39.90 & 53.85 & 42.03&\textbf{57.31}\\
& {Kappa$\times$100} & 54.26 & 49.11 & 49.09 & 50.10 & 55.24 & 47.98 & 47.09 & 52.19 & \textbf{57.14} \\
\midrule[0.5pt]
\bottomrule[1.3pt]
\end{tabular}
\label{tab:comparisonhyrank}
\end{table*}

\begin{table*}[htbp]
\centering
\caption{SAMPLE SIZE OF DATASETS}
\label{sizehouston}
\resizebox{\textwidth}{!}{%
\begin{tabular}{@{}lccccccccc@{}}
\toprule[1.3pt]
\midrule[0.5pt]
No. & Class     & Pavia University & Pavia Center & Class     & Houston2013 & Houston2018 & Class     & Dioni & Loukia \\ \midrule
C1  & Trees     & 3064             & 7598         & Grass healthy            & 345         & 1353        & Dense Urban Fabric             & 1262  & 206    \\
C2  & Asphalt   & 6631             & 9248         & Grass stressed           & 365         & 488         & Mineral Extraction Sites       & 204   & 54     \\
C3  & Bricks    & 3682             & 2685         & Trees                    & 365         & 2766        & Non Irrigated Arable Land      & 614   & 426    \\
C4  & Bitumen   & 1330             & 7287         & Water                    & 285         & 22          & Fruit Trees                    & 150   & 79     \\
C5  & Shadow    & 947              & 2863         & Residential buildings    & 319         & 5347        & Olive Groves                   & 1768  & 1107   \\
C6  & Meadows   & 18649            & 3090         & Non-residential buildings & 408         & 32459       & Coniferous Forest              & 361   & 422    \\
C7  & Bare soil & 5029             & 6584         & Road                     & 443         & 6365        & Dense Sclerophyllous Vegetation& 5035  & 2996   \\ 
C8 &  &              &          &                      &          &           & Sparce Sclerophyllous Vegetation& 6374 & 2361   \\
C9 &  &              &          &                      &          &            & Sparsely Vegetated Areas       & 1754  & 399    \\
C10 &  &              &          &                      &          &           & Rocks and Sand                 & 492   & 453    \\
C11 &  &              &          &                      &          &           & Water                          & 1612  & 1393   \\
C12 &  &              &          &                      &          &           & Coastal Water                  & 398   & 421    \\
\midrule
    & Total     & 39332            & 39355        & Total                    & 2530        & 53200       & Total                          & 20024 & 10317  \\ \midrule[0.5pt]
\bottomrule[1.3pt]
\end{tabular}%
}
\end{table*}

Our experiments are basically implemented on \domainbed \cite{domainbed}, a popular and widely accepted testbed with a multitude of methodologies and benchmarks for domain generalization testing and evaluating.

\begin{figure*}[htbp]
\centering
\includegraphics[width=\linewidth]{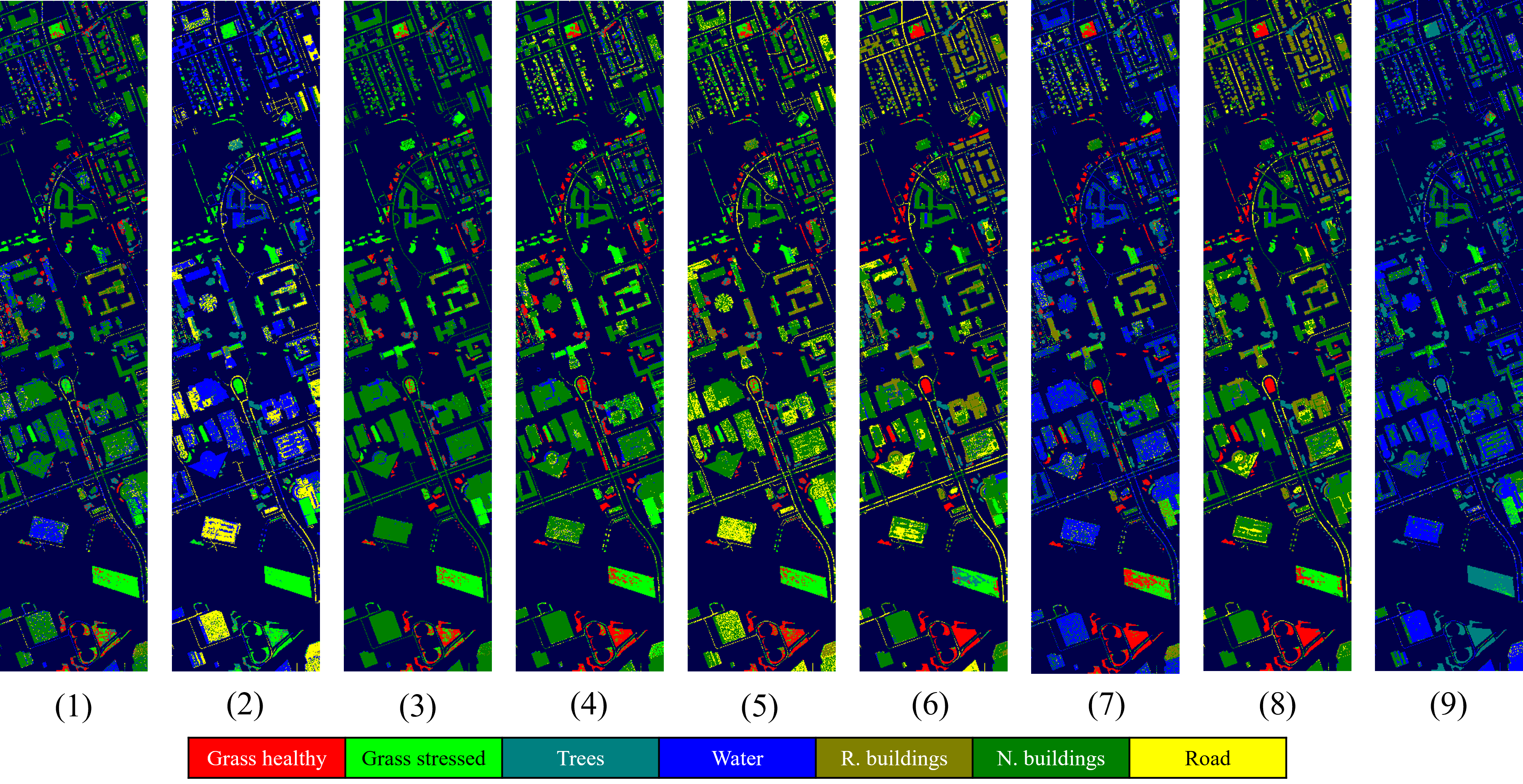}
\caption{Classification results on cross-scene Houston. (1)ERM. (2)CAD. (3)MTL. (4)CausIRL. (5)SelfReg. (6)SDEnet\_spe. (7)LiCa. (8)S${^2}$ECnet\_spe. (9)C$^3$DG.}
\label{fig:hu_exp}
\end{figure*}

\begin{figure*}[htbp]
\centering
\includegraphics[width=\linewidth]{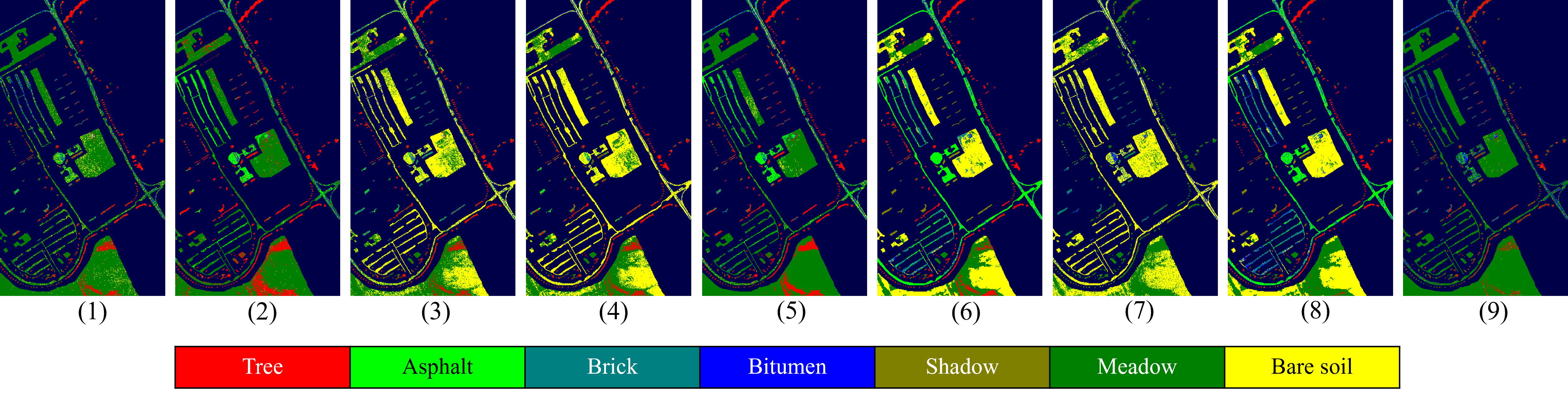}
\caption{Classification results on cross-scene Houston. (1)ERM. (2)CAD. (3)MTL. (4)CausIRL. (5)SelfReg. (6)SDEnet\_spe. (7)LiCa. (8)S${^2}$ECnet\_spe. (9)C$^3$DG.}
\label{fig:cp_exp}
\end{figure*}

\begin{figure*}[htbp]
\centering
\includegraphics[width=\linewidth]{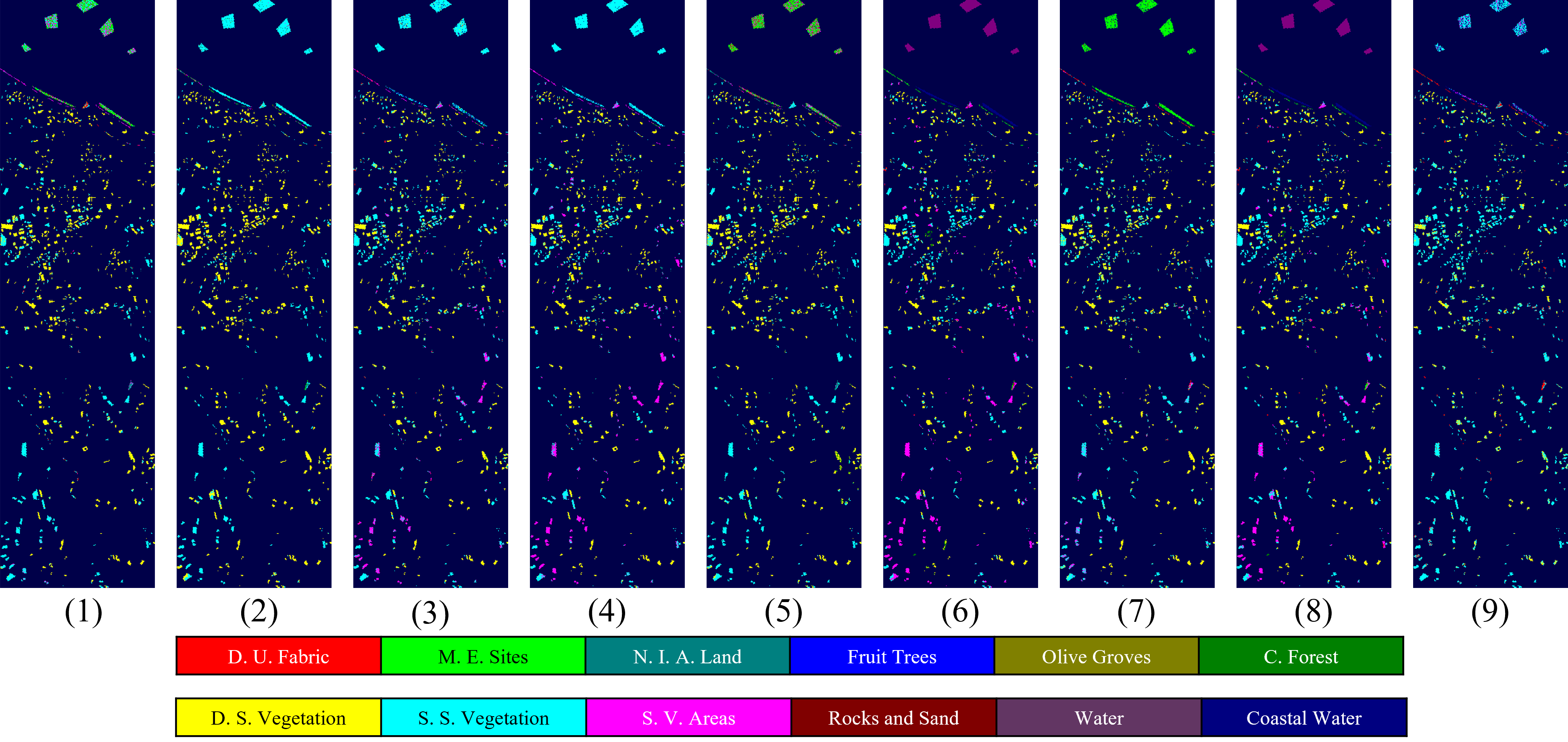}
\caption{Classification results on cross-scene Houston. (1)ERM. (2)CAD. (3)MTL. (4)CausIRL. (5)SelfReg. (6)SDEnet\_spe. (7)LiCa. (8)S${^2}$ECnet\_spe. (9)C$^3$DG.}
\label{fig:hr_exp}
\end{figure*}

In this subsection, we will initially compare the proposed method in this paper with the emperical risk minimization (ERM) method and various DG methods such as CAD \cite{dubois2021optimal}, CausIRL \cite{chevalley2022invariant}, MTL \cite{mtl}, SelfReg \cite{kim2021selfreg} and VREx \cite{krueger2021out}, to demonstrate its superiority for HSI classification tasks. Moreover, we adopt two single-source generation HSI classification methods SDEnet \cite{single-DG} and S$^2$ECnet \cite{s2ecnet} and a multi-source HSI classification method LiCa \cite{lica} for comparison, to illustrate the rationality of utilizing domain information in the proposed method. 
For a fair comparison, we removed the spatial extraction modules from methods SDEnet\cite{single-DG} and S$^2$ECnet \cite{s2ecnet}. Specifically, we set the patch size to one and eliminated the spatial randomization in SDEnet and the spatial encoder-decoder architecture in S$^2$ECnet. In the comparative experiment table, these modifications are denoted as SDEnet\_spe and S$^2$ECnet\_spe, respectively. The superior performance illustrates a promising approach to addressing the phenomenon of hyperspectral-monospectra.


To construct the multi-source domains, we adopt a straightforward and intuitive method to divide each dataset into four subpictures by splitting the image into four along the central point. This partitioning approach is not only easy to implement but also aligns well with the observations of HSI described in \cite{single-DG}. In HSI classification tasks, the evaluation metrics that are commonly employed include Overall Accuracy (OA), Average Accuracy (AA) across each class, and the Kappa Coefficient. These benchmarks are also utilized in our experiments to assess the performance of the proposed method for comparison. The analysis for the results on different datasets is as follows.

\noindent\textbf{Cross-Scene Houston:}
As illustrated in Table \ref{tab:comparisonhou}, we bold the best score among the comparison models. The direct application of domain generalization method SelfReg for natural scenes shows minimal advantage compared to ERM, with some methods even exhibiting significant disadvantages, such as CAD, CausIRL, and MTL. Domain generalization methods that include finer spectral extractors generally show a common advantage. Among these, our method demonstrates priority on multiple benchmarks.

\noindent\textbf{Cross-Scene Pavia:}
As illustrated in Table \ref{tab:comparisonpa}, we have highlighted the best score among the comparison models in bold. All the domain generalization methods demonstrate advantages in OA compared to the baseline ERM. Among these methods, our approach shows superior performance on multiple benchmarks.

\noindent\textbf{Hyrank Datasets:}
As illustrated in Table \ref{tab:comparisonhyrank}, we have highlighted the best score among the comparison models in bold. Our method shows advantages across the three benchmarks: OA, AA, and the Kappa Coefficient.

From Table \ref{tab:comparisonhou}, \ref{tab:comparisonpa}, \ref{tab:comparisonhyrank}, the domain generalization methods for 3-channel natural scene images generally perform worse than those for HSI. Additionally, when the spatial extraction capability is removed, these single-source domain methods typically exhibit inferior performance compared to our proposed method.

\subsection{Ablation experiment}

\begin{table}[htbp]
\centering
\captionsetup{font=small}
\caption{ABLATIONS ON CROSS-SCENE HOUSTON}
\label{tab:method_comparison}
\begin{tabular}{m{0.9cm}m{1.0cm}m{1.36cm}m{1.0cm}m{1.1cm}m{1.1cm}}
\toprule[1.3pt]
\midrule[0.5pt]
 method & \footnotesize Context Network & \footnotesize Conditional Distribution& \footnotesize VAE & \footnotesize CRIB& \footnotesize Accuracy  \\
\midrule
ERM & $\times$ & $\times$ & $\times$&$\times$ & 66.48 \\
1-light & $\checkmark$ & $\times$ & $\times$&$\times$ & 52.68\\
1-vae & $\checkmark$ & $\times$ & $\checkmark$&$\times$ & 54.52\\
$C$-vae & $\checkmark$ & $\checkmark$ & $\checkmark$&$\times$ & 67.23\\
C$^3$DG & $\checkmark$ & $\checkmark$ & $\checkmark$& $\checkmark$ & \textbf{68.23}\\
\midrule[0.5pt]
\bottomrule[1.3pt]
\label{ablation}
\end{tabular}
\end{table}

To validate the effectiveness of each block, we conduct experiments as illustrated in Table \ref{ablation}. ERM, as previously described, serves as a baseline and is simply a feature extractor and a classifier. Initially, we form an approach with a context encoder that consists of a shallow 1d-CNN feature extractor (i.e., 1-light in Table \ref{ablation}). Subsequently, we contrast it with two methods of which context networks that utilize VAE structure. The first employs a single context branch (i.e., 1-vae in Table \ref{ablation}), and the second uses multiple context branches, specifically C branches (i.e., C-vae in Table \ref{ablation}).


The experimental results in Table \ref{ablation} first confirm the effectiveness of C$^3$DG for addressing hyperspectral-monospectra, and also indicate that considering conditional distributions fits more closely with this strategy. Additionally, the comparison between 1-vae and 1-light reveals the success of implementing more complex architecture. Moreover, The comparison between C-vae and C$^3$DG demonstrates that our method retains classification efficacy while simplifying model parameters.

\section{Conclusion}
This work proposed a conditional domain generalization adapting-inference strategy tailored for HSI Classification. Targeting the phenomenon of hyperspectral-monospectra (i.e., same spectra but different material), our proposed method is capable of generating revised outputs depending on the distribution feature (i.e., the aforementioned domain features) of the input batch. This approach offers a solution that, on the one hand, reduces the impact of hyperspectral-monospectra in pixel-wise HSI classification tasks, and on the other hand, provides an adaptive strategy that can be applied to any HSI classification methods with a series of theorems for the rationality of addressing hyperspectral-monospectra. Extensive experiments conducted on three datasets confirm the efficacy of the proposed methods in the strategy of test time revised output. The performance delivered by C$^3$DG is on par with, or in some cases exceeds, that of current single-source domain generalization HSI classification methods.

\IEEEpubidadjcol
\bibliographystyle{ieeetr} 
\bibliography{IEEEcite} 

\end{document}